\documentclass[runningheads,a4paper]{llncs}

\usepackage{amssymb}
\usepackage{graphicx}
\usepackage{array}
\usepackage{graphicx}
\usepackage{xcolor}
\usepackage{color}
\usepackage{amssymb}
\usepackage{amsmath}
\usepackage{inferance}
\usepackage{deduc}
\renewcommand{\leq}{\mathrel{\leqslant}}
\renewcommand{\geq}{\mathrel{\geqslant}}

\usepackage[linesnumbered,algoruled]{algorithm2e}

\usepackage{multicol}
\usepackage[english]{babel}
\usepackage{blindtext}
\usepackage{url}
\urldef{\mailsa}\path|{jabbour,boudane,sais}@cril.fr, badran.raddaoui@telecom-sudparis.eu|    
\newcommand{\keywords}[1]{\par\addvspace\baselineskip
\noindent\keywordname\enspace\ignorespaces#1}

\begin{document}

%\mainmatter  % start of an individual contribution

\title{Efficient Encodings of Conditional Cardinality Constraints}

%\titlerunning{Jabbour et al.}

%\author{}

%\institute{}

\author{  $^1$Abdelhamid Boudane \and $^1$Said Jabbour  \and $^2$Badran Raddaoui \and $^1$Lakhdar Sais}
%Abdelhamid Boudane\and Said Jabbour\and Lakhdar Sais\and Yakoub Salhi\\

\institute{$^1$CRIL-CNRS UMR 8188, Universit\'{e} d'Artois, F-62307 Lens Cedex, France\\ 
%$^2$ TSP - T\'el\'ecom SudParis, 91011 \'Evry Cedex - France\\ 
$^2$ SAMOVAR, T\'el\'ecom SudParis, CNRS, Univ. Paris-Saclay, Evry, France 
\mailsa}


%\toctitle{Lecture Notes in Computer Science}
%\tocauthor{Authors' Instructions}
\maketitle
\vspace*{-0.5cm}
\begin{abstract}

%Encoding real-world problems into propositional satisfiability is as fundamental as  solving 
In the encoding of many real-world problems to propositional satisfiability, the cardinality constraint is a recurrent constraint that needs to be managed effectively.
Several efficient encodings have been proposed while missing that such a constraint can be involved in a more general propositional formulation.
To avoid combinatorial explosion, Tseitin principle usually used to translate such general propositional formula to Conjunctive Normal Form (CNF), introduces fresh propositional variables to represent sub-formulas and/or complex contraints.
Thanks to Plaisted and Greenbaum improvement, the polarity of the sub-formula $\Phi$ is taken into account leading to conditional constraints of the form $y\rightarrow \Phi$, or  $\Phi\rightarrow y$, where $y$ is a fresh propositional variable.
In the case where $\Phi$ represents a cardinality constraint, such translation leads to conditional cardinality constraints subject of the present paper.
We first show that when all the clauses encoding the cardinality constraint are augmented with an additional new variable, most of the well-known encodings cease to maintain the generalized arc consistency property. Then, we consider some of these encodings and show how they can be extended to recover such important property.
An experimental validation is conducted on a SAT-based pattern mining application,  where such conditional cardinality constraints is a cornerstone, showing the relevance of our proposed approach.

\end{abstract}

\section{Introduction}
\label{sect:introduction}

Continuous improvements in SAT solver technology have resulted in a real scaling up and widening of the class of real-world problems that can be solved in practice. 
The modeling phase of such increasing number and  more complex applications into propositional formulas in Conjunctive Normal Form (CNF) suitable for use by a satisfiability solver becomes even more crucial.
The modeling issue follows several polynomial transformations and rewriting steps, starting from high level description, using high order language or full propositional logic, to low level formulation, usually a formula in CNF.
The whole process preserves propositional satisfiability, thanks to the extension principle  \cite{Tseitin68}, allowing the introduction of new variables to represent sub-formulas or complex constraints.
Among such constraints, \textit{cardinality} and \textit{pseudo-boolean constraints}, expressing numerical bounds on discrete quantities, are the most popular as they arise frequently
in the encoding of many real-world problems including scheduling, logic synthesis or verification, product configuration and data mining. For the above reasons, there have been various approaches  addressing  the issue of finding an efficient encoding of cardinality  (e.g. \cite{Warners96,BailleuxB03,Sinz05,SilvaL07,AsinNOR09,JabbourSS14}) and pseudo-boolean constraints (e.g. \cite{EenS06,BailleuxBR09}) as a CNF formula. Efficiency refers to both the compactness of the representation (size of the CNF formula)  and to the ability to achieve the same level of constraint propagation (generalized arc consistency) on the CNF formula. However, most of the proposed encodings does not take care of its interactions with the remaining part of the propositional formula, through different logical connectives. To avoid combinatorial explosion, the Tseitin principle \cite{Tseitin68} is usually used to translate general propositional formula to CNF, making use of fresh propositional variables to represent sub-formulas and/or complex contraints.
Thanks to Plaisted and Greenbaum  \cite{Plaisted:1986} improvement, the polarity of the sub-formula $\Phi$ is taken into account leading to conditional constraints of the form $y\rightarrow \Phi$ or  $\Phi\rightarrow y$, where $y$ is a fresh propositional variable. When a cardinality constraint is involved as a sub-formula, such translation leads to what we call a \textit{conditional} cardinality constraint.

The translation of {\it single} cardinality or pseudo-boolean constraints to SAT is a well studied problem.
We are aware of only one contribution that consider the interactions of such constraints with the remaining part of the formula involving it.
Indeed, in \cite{Ignasi2015}, the authors described how the encoding of linear constraints can be improved by taking implication chains appearing in the formula into account.
The resulting encodings are smaller and can propagate more strongly than separate encodings. 
 
In this paper, we  introduce a novel variant of cardinality constraints, called conditional cardinality constraints,  defined as $y \rightarrow \sum_{i=1}^{n} x_i \le k$. It expresses that no more than $k$ variables can be set to $true$, when setting the condition $y$ to $true$.
We first show that by adding $\neg y $ disjunctively to all the clauses resulting from the encoding of the cardinality constraint, most of the well-known encodings cease to maintain constraint propagation.  We then address the issue of extending such encodings while maintaining generalized arc consistency. We also consider the particular case of conditional AtMostOne constraints, i.e., $k=1$. Experimental evaluation is conducted on a SAT based non redundant association rules mining problem, showing the relevance of our proposed framework. 
%is conducted
%Interestingly, we show how to compact a combination of conditional AtMostOne constraints by exploiting the common patterns of variables. 
%\\

%The rest of this paper is organized as follows. After some preliminary definitions and technical background, we present three well known Cardinality based encodings, including Pigeon-Hole, Sorted networks and sequential counter based encodings. Then we present how such encodings can be extended to deal both with conditional AtMostOne constraints and conditional cardinality constraints, while preserving generalized arc consistency by unit propagation. Then, we provide an extensive experimental evaluation on real world SAT instances encoding itemset mining problems.  

\section{Technical Background and Preliminary Definitions}\label{sec:tb}

\subsection{Preliminary Definitions and Notations}
Let ${\mathcal L}$ be a propositional language of formulas built in the standard
way, using usual connectives ($\vee$, $\wedge$, $\neg$, $\rightarrow$, $\leftrightarrow$) and a set of propositional variables. 
A propositional formula ${\cal F} $ in CNF is a conjunction of {\it clauses}, where a clause is 
a disjunction of {\it literals}. 
A literal is a positive ($x$) or negated ($\neg{x}$) 
propositional variable.  A clause can be represented as a set of literals and a formula as a set of clauses. 
The two literals $x$ and $\neg x$ are called {\it complementary}. 
We note $\tilde{l}$  the complementary literal of $l$. 
For a set of literals $L$, $\bar{L}$ is defined as $\{\tilde{l} ~|~ l \in L\}$. For a clause $c$, we note $\bar{c}=\bigwedge_{l\in c} \tilde{l}$. 
%The two literals $x$ and $\neg{x}$ are called {\it complementary}. 
%We denote by $\tilde{l}$  the complementary literal of $l$. More precisely, if  $l = x$ then $\tilde{l}=\neg x$, otherwise $\tilde{l}=x$. 
%For a set of literals $L$, $\tilde{L}$ is defined as $\{\tilde{l} ~|~ l \in L\}$. 
%The variable associated to a literal $l$ is denoted by $|l|$. 
% A CNF formula can also be seen as a set of clauses, and a clause as a set of literals. 
%The size of the CNF ${\cal F}$ is defined as $\sum_{c\in {\cal F}} |c|$, where $|c|$ is  the number of literals in $c$. 
A {\it unit} clause is a clause containing only one literal (called {\it unit literal}), while a  binary clause contains exactly two literals.  A Horn (resp. reverse Horn) clause is a clause with at-most one positive (resp. negative) literal. A positive (resp. negative) clause is a clause whose literals are all positive (resp. negative). 
An \emph{empty clause}, denoted $\perp$, is interpreted as  false (unsatisfiable), whereas an \emph{empty CNF formula}, denoted $\top$, is interpreted as true (satisfiable).

Let us recall that any general propositional formula can be translated to CNF using linear Tseitin's encoding \cite{Tseitin68}.
This can be done by introducing fresh variables to represent sub-formulas in order  to 
represent their truth values. For example, given a propositional formula containing the variables $x$ and $y$, and $\alpha$  is a fresh  variable, one can add the definition $\alpha\leftrightarrow x\vee y$ (called \textit{extension}) to the formula while preserving satisfiability.
Two decades later, after Tseitin's seminal paper, Plaisted and Greenbaum presented an improved CNF translation  that essentially produces a subset of Tseitin's representation \cite{Plaisted:1986}.
The authors noticed that by keeping track of polarities of sub-formulas, one can remove large parts of Tseitin translation. For example, when the disjunction $x\vee y$ is a sub-formula with positive polarity, it is sufficient to add the formula $\alpha\rightarrow x\vee y$, i.e., a clause $(\neg \alpha\vee x\vee y)$. 

The set of variables occurring in ${\cal F}$ is denoted $V_ {\cal F}$ and its associated set of literals 
\textcolor{black}{$L_{\cal F}=\cup_{x\in{\cal V}_{\cal F}} \{x,\neg x\}$}.
A set of literals is \emph{complete} if it contains one literal for each variable in $V_{\cal F}$,
and \emph{fundamental} if it does not contain complementary literals. A literal $l$ is called {\it monotone or pure} if $\tilde{l}$ does not appear in ${\cal F}$.  
An {\it interpretation} $\rho$ of a formula ${\cal F}$ is a function which  associates a truth value $\rho(x)\in\{0, 1\}$ ($0$ for false and $1$ for $true$)
to some of the variables $x \in V_ {\cal F}$. 
$\rho$ is \emph{complete} if it assigns a value to every $x \in V_ {\cal F}$, and \emph{partial} otherwise.
An interpretation is alternatively represented by a complete and fundamental set of literals.
A {\it model} of a formula ${\cal F}$ is an  interpretation $\rho$ that  satisfies the  formula, denoted $\rho\models\cal F$. A formula ${\cal G}$ is a logical consequence of a formula ${\cal F}$, denoted ${\cal F}\models{\cal G}$, iff every model of ${\cal F}$ is a model of ${\cal G}$. The {\it SAT problem} consists in deciding if a given CNF formula admits a model or not.

 ${\cal F}|_x$ denotes the formula obtained from ${\cal F}$ by assigning $x$ the truth-value $true$. Formally, ${\cal F}|_x = \{c ~|~ c\in {\cal F}, \textcolor{black}{\{x, \neg{x}\} \cap c} = \emptyset\} \cup
\{\textcolor{black}{c\backslash \{\neg{x}\}} ~|~ c\in {\cal F},  \neg{x}\in c\}$. 
This notation is extended to interpretations: given an interpretation $\rho=\{x_1,\dots, x_n\}$, 
we define ${\cal F}|_\rho= (\dots (({\cal F}|_{x_1})|_{x_2})\dots |_{x_n})$. ${\cal F}^*$ denotes the formula ${\cal F}$ closed under unit propagation, defined recursively as follows:
(1) ${\cal F}^* = {\cal F}$ if ${\cal F}$ does not contain any unit clause, 
(2)  ${\cal F}^* = \perp$ if ${\cal F}$ contains two \textcolor{black}{unit-clauses $\{x\}$ and  $\{\neg{x}\}$},
(3) otherwise, ${\cal F}^*= ({\cal F}|_{x})^*$ where $x$ is the
literal appearing in a unit clause of ${\cal F}$.  A clause $c$ is deduced by unit propagation from ${\cal F}$, noted ${\cal F}\models^* c$, iff $({\cal F}\wedge\bar{c})^* =\bot$.

%Let $c_i$ and $c_j$ be two clauses such that $c_i = (x\vee\alpha)$ and $c_j = (\neg x\vee\beta)$, $\eta[x, c_i, c_j] = (\alpha\vee\beta)$ denotes the \emph{resolvent} on $x$ between $c_i$ and $c_j$. A resolvent is called {\it tautological} when it contains complementary literals. 

%\marginpar{assign literals to true, not variables, interpretation
 % becomes associated fundamental set}
 
%Let $c_1$ and $c_2$ be two clauses of a formula ${\cal F}$. We say that $c_1$ (resp. $c_2$) 
%subsume (resp. is subsumed) $c_2$ (resp. by $c_1$) iff  $c_1 \subseteq c_2$. If $c_1$ 
%subsume $c_2$, then   $c_1 \models c_2$ (the converse is not true). 
%Also ${\cal F}$ and ${\cal F} -{c_2}$ 
%are equivalent with respect to satisfiability. 
%
%Let $c\in{\cal F}$ such that $x\in c$, the literal $x$ of $c$ is called {\it blocked} if  $\forall c'\in{\cal F}$ such that  $\neg x\in c'$ and $c\neq c'$, $\eta[x, c, c']$ is a tautology. 
%A clause $c\in{\cal F}$ is a {\it blocked clause} if it contains a blocked literal \cite{Kullmann97ona}.   A blocked clause $c\in{\cal F}$ can be deleted from ${\cal F}$ while preserving satisfiability. 
% 

\subsection{CNF Encodings of Cardinality Constraints: An Overview}
\label{sec:overview}
\subsubsection{Pigeon-Hole based Encoding:}
\label{sec:phc}
In \cite{JabbourSS14}, the authors proposed a new encoding of the cardinality constraints $\sum_{i=1}^{n} x_i \ge k$, based on the Pigeon-Hole principle. 
They observed that the semantic of the cardinality constraint can be equivalently expressed as the problem of putting $k$ pigeons into $n$ holes. 
The first formulation, called ${\cal P}_n^{k}$, given in \cite{JabbourSS14}, is simply expressed by the following set of constraints: 
\begin{equation}
\label{eqc1}
\bigwedge_{j=1}^{k} (\neg p_{ji} \vee x_i), ~~~1\leq i\leq n 
\end{equation}
\begin{equation}
\label{eqc2}
  \bigvee_{i=1}^n p_{ji}, ~~~1\leq j\leq k, \hspace{1cm} \bigwedge_{1\leq j < j'\leq b} ( \neg p_{ji}\vee\neg p_{j'i}), ~~~ 1\leq i\leq n
\end{equation}
%\begin{equation}
%\label{eqc3}
%\bigwedge_{1\leq j < j'\leq b} ( \neg p_{ji}\vee\neg p_{j'i}), ~~~ 1\leq i\leq n
%\end{equation}
The equations (\ref{eqc2})  encode the well-known pigeon hole problem $PHP_n^{k}$, where $k$ is the number of pigeons and $n$ is the number of holes ($p_{ji}$ expresses that  pigeon $j$ is in hole $i$).  
%The mapping between the models of $PHP_n^{b}$ and those of $\sum_{i=1}^n x_i \geq b$ are obtained thanks to the equation \ref{eqc1}. Indeed, the propositional variable $x_i$ is true if  the hole $i$ contains one of the pigeons $k$ for $1\leq k\leq b$. If we take again the previous counter example, the CNF formula ${\cal P}_n^{b}$ becomes unsatisfiable as it encodes an unsatisfiable Pigeon-Hole problem $PHP_{b-1}^{b}$.  
%The overall size of the encoding is in ${\cal O}(k\times n)$ variables and ${\cal O}(n\times k^2)$ clauses. 
Unfortunately, checking the satisfiability of a Pigeon-Hole formula is computationally hard.
%except if  we use resolution with symmetry or extended resolution proof systems. 
To maintain generalized arc consistency (GAC), the authors proposed an improvement obtained by breaking the symmetries between the variables $p_{ij}$ involved in the pigeon hole expression (equations (\ref{eqc2}) and (\ref{eqc3})).
By resolution between the clauses of  symmetry breaking predicates and those of ${\cal P}_n^{k}$, the authors derived the following  encoding, called $ph{\cal P}_n^{k}$: 
% in ${\cal O}(k\times (n-k))$ variables and clauses. 
%For clarity reason, and to better visualize the reductions on our previous encoding ${\cal P}_n^{b}$, we use the following matrix  representation  for the CNF formula (\ref{eqc2}). Each row represents a positive clause of (\ref{eqc2}). 
%
%{\small
%\noindent
%%\begin{equation}
%%\label{matrice}
%\[
%   \left (
%   \begin{array}{llllllr}
%      p_{11}   & \cdots & & [p_{1b}&~ \cdots & p_{1n}]\\
%      p_{21}&  \cdots &~~~~~~~~~~ [p_{2(b-1)} & ~~~~~\cdots &   ~~p_{2(n-1)}]&p_{2n} \\
%      &  & & & &  \\  
%  %    \vdots    & &~\iddots{}  & & \iddots{} & \vdots \\
%      &  & & & &  \\  
%      &   &  && & \\% p_{{b-1}{(n-b+3)}}    &\\
%     $[$ p_{{b}1}  &  \cdots & p_{{b}{(n-(b-1))}}$]$ &&\cdots &  p_{{b}n} \\
%   \end{array}
%   \right )
%\]
%%\caption{Toto}
%%\end{equation}
%}
%
%The authors gives a enhanced CNF Pigeon-Hole based encoding, called $ph{\cal P}_n^{b}$, of  a cardinality constraint:
% 
\begin{equation}
\label{form:pigeonMP}
 	\bigwedge_{1\le i \le k} (\bigvee_{1 \le j \le n-k+1} p_{ij})
 \end{equation}
\begin{equation}
\label{form:pigeonXP}
 	\bigwedge_{1 \le i \le k} \bigwedge_{1 \le  j \le n-k+1} (x_{i+j-1} \vee \neg p_{ij})
 \end{equation}
\begin{equation}
\label{form:pigeonPP}
 	\bigwedge_{1 \le i < k}  \bigwedge_{1 \le j < n-k+1}  (\neg p_{(i+1)j} \vee \bigvee_{1 \le l \le j} p_{il})
 \end{equation}
\begin{example}
\label{ex:ph}
Let us consider the inequality $x_1 + x_2  + x_3 + x_4 + x_5 + x_6 \ge 4$. Using the pigeon-hole based encoding $ph{\cal P}_6^{4}$, we obtain the following CNF:\\
%\[
%   \left (
%   \begin{array}{llllll}
%                      
%                                  &             &              & ~~~p_{41}   &  ~~~p_{42}  & ~~~ p_{43} \\
%                                  &                     & ~~~p_{31}  &  ~~~p_{32}   &  ~~~p_{33}   &          \\
%                                  &  ~~~p_{21}  & ~~~p_{22}  &  ~~~p_{23}    &              & \\
%                     p_{11}  &  ~~~p_{12}  &  ~~~p_{13}  &              &                &\\
%                     
%                    % \multicolumn{2}{l}{\underbrace{\hspace{1.2cm}}_{x_1} } & & \multicolumn{3}{l}{\underbrace{\hspace{2cm}}_{x_2} } 
%                     
%   \end{array}
%   \right )
%\]
\begin{minipage}[b]{0.33333\textwidth}
\raggedright

\[
  % \left (
   \begin{array}{ll}
   	p_{11}  \vee \neg p_{21}   \\
	   p_{11}  \vee p_{12} \vee \neg p_{22} \\
   	p_{21}  \vee \neg p_{31}   \\
	   p_{21}  \vee p_{22} \vee \neg p_{32} \\
   	p_{31}  \vee \neg p_{41}   \\
	   p_{31}  \vee p_{32} \vee \neg p_{42} \\
  \end{array}
  % \right )
\]
\end{minipage}%
\begin{minipage}[b]{0.33333\textwidth}
\centering
\[
  % \left (
   \begin{array}{l}
x_1 \vee \neg p_{11} ~~~~~~~~ x_2\vee  \neg p_{12} \\
x_2 \vee \neg p_{21}  ~~~~~~~~ x_3 \vee \neg p_{13} \\
x_3 \vee \neg p_{22}  ~~~~~~~~ x_3 \vee \neg p_{31}  \\
x_4 \vee \neg p_{23}  ~~~~~~~~ x_4 \vee \neg p_{32}  \\
x_4 \vee \neg p_{41}  ~~~~~~~~ x_5 \vee \neg p_{33}  \\
x_5 \vee \neg p_{42}  ~~~~~~~~ x_6 \vee \neg p_{43} \\
  \end{array}
  % \right )
\]
\end{minipage}%
\begin{minipage}[b]{0.33333\textwidth}
\raggedleft
\[
  % \left (
   \begin{array}{l}
    p_{11}   \vee p_{12} \vee p_{13}   \\
    p_{21}   \vee p_{22}   \vee p_{23}  \\
    p_{31}  \vee p_{32}   \vee p_{33}   \\
    p_{41}  \vee p_{42}  \vee p_{43}   \\
  \end{array}
  % \right )
\]
\end{minipage}
\end{example}

 \subsubsection{Sorting Networks based Encoding:}
 \label{sec:SNE}
 One of the most effective encodings for cardinality constraints is based on sorting networks \cite{EenS06}. In this encoding, the cardinality constraint $\sum_{i=1}^n x_i \leq k$ is translated into a single sorter with 
 $n$ inputs $X=\{ x_1, \dots,  x_n\}$ and $n$ outputs $Z=\{z_1,  \ldots , z_n\}$ (sorted in descending order) where the k$^{th}$ output is forced to $true$. 
% One of the most effective encodings for cardinality constraints is based on sorting networks \cite{EenS06}. It takes as input a set of $n$ propositional variables $ x_1, \dots,  x_n$ and outputs $y_1, . . . , y_n$ propositional variables, such that if exactly k input variables are true, then $y_1, . . . , y_k$ become true and $y_{k+1}, . . . , y_n$ become false. 
 The idea behind this encoding is to sort the input variables into true variables followed by false variables. To satisfy the constraint $\sum_{i=1}^n x_i \leq k$, it is sufficient to set $z_{k+1}$ to false. 
In \cite{EenS06}, the authors proved that the sorting networks based encoding maintains generalized arc consistency. Let us note $\Phi_{SN}^{n, k}(X;Z)$ the formula representing the sorting networks based circuit that takes as input the set of propositional variables $X$ and outputs an unary number represented by the set of propositional variables $Z$. The following formula defines the encoding: 
  %of the cardinality constraint $\sum_{i=1}^n x_i \leq k$.
  \begin{equation}
  \label{sne}
  \Phi_{SN}^{n, k}(X;Z) \wedge \neg z_{k+1} 
\end{equation} 
As the outputs $Z$ are sorted in descending order, by fixing $z_{k+1}$ to $false$, all the remaining variables $z_{k+2},\dots, z_{n}$ must be propagated to $false$. Consequently, as the output variables are sorted in descending order, at most $k$ variables $z_1, \dots, z_k$ might be assigned to $true$. 
Let us note that the formula $\Phi_{SN}^{n, k}(X;Z)$ encoding the sorting network is a horn formula, derived using a basic comparator between two propositional variables \cite{EenS06}. Given two propositional variables  $x_1$ and $x_2$ from $X$, the comparator outputs two  variables $z_1$ and $z_2$ from $Z$, the two comparator, noted $2$-$comp(x_1,x_2; z_1,z_2)$, is defined by the following horn formula:
\begin{equation}  
x_1 \rightarrow z_1 \; ; \; x_2 \rightarrow z_1 \; ; \; x_1\wedge x_2 \rightarrow z_2  
\end{equation}
This formula allows to sort the two variables $x_1$ and $x_2$ resulting in two other variables $z_1$ and $z_2$ in descending order.
For example, when $x_1$ (resp. $x_2$) is assigned to $false$ (resp. $true$), the output variable $z_1$ (resp. $z_2$) is assigned to $true$ (resp. $false$).  
%This formula can be written in a clausal form as $( \neg x_1 \vee z_1)\wedge (\neg x_2 \vee z_1) \wedge (\neg x_1 \vee \neg x_2 \vee z_2)$. 
%The second constraint ,represented by the unit literal $\neg s_{k+1}$, allows us to express the that we can have at most $k$ true variables as the variables are sorted before.\\
For more details, we refer the reader to \cite{EenS06} and \cite{AsinNOR11}.  

\subsubsection{Sequential Unary Counter based Encoding:}
 The sequential counter based encoding of a cardinality constraint proposed by Carsten Sinz in \cite{Sinz05} is another well-known encoding that preserves the generalized arc consistency property.  
 %The following constraints represent the sequential counter based encoding of $\sum_{i=1}^{n} x_i \leq k$. 
It computes for each propositional variable $x_i$, \textcolor{black}{the partial sums} $s_i=\sum_{j=1}^{i} x_j$ for increasing values of $i$ up to the final $i = n$. The values of all the sums are represented as unary numbers of size equals to $k$. The encoding is defined as follows: 
\begin{equation}
\label{eq1}
(\neg x_1 \vee s_{1,1})
\end{equation} 
\begin{equation}
\label{eq2}
\bigwedge_{1 < j \leq k} \neg s_{1,j}
\end{equation}  
\begin{equation}
\label{eq3}
\bigwedge_{1 < i < n} (\neg x_i \vee s_{i,1})\wedge (\neg s_{i-1,1} \vee s_{i,1})
\end{equation} 
\begin{equation}
\label{eq4}
\bigwedge_{1 < i < n} \bigwedge_{1<j\leq k} (\neg x_i \vee \neg s_{i-1,j-1} \vee s_{i,j} )\wedge(\neg s_{i-1,j} \vee s_{i,j})
\end{equation}
\begin{equation}
\label{eq5}
\bigwedge_{1 < i \leq n} (\neg x_i \vee \neg s_{i-1,k})
\end{equation} 

The variables \textcolor{black}{$s_{i,j}$} denotes the $j^{th}$ digit of the  $i^{th}$ partial sum $s_i$ in unary representation. The constraints (\ref{eq1}) and (\ref{eq2}) correspond to the case $i=1$. 
The formula (\ref{eq5}) is very important.
It allows to detect the inconsistency and preserves the GAC property at the same time.
The other constraints allow the propagation of any changes of a partial sum \textcolor{black}{$s_i$} after any assignment of $x_i$ variables.
Let us note that the formula derived by the sequential unary counter based encoding is also a horn formula. 
% \subsubsection{Sorted Network Based Encoding}
%A faire \cite{EenS06}
% \subsubsection{Sequential Unary Counter Based Encoding}
%
%A faire \cite{Sinz05}
\section{Conditional Cardinality Constraints Encodings}

In this section, we show how the cardinality constraint encodings of Section \ref{sec:overview}, can be effectively extended to encode conditional cardinality constraints of the form $y\rightarrow \sum_{i=1}^{n} x_i \leq k$ while preserving generalized arc consistency maintained by unit propagation. More precisely, for such conditional cardinality constraint, maintaining GAC, means that when $y$ is assigned the truth value $true$, the encoding must maintain GAC on  the cardinality constraint $\sum_{i=1}^{n} x_i \leq k$. On the other hand, when the cardinality constraint is $false$ under the current assignment, the variable $\neg y$ must be deduced by unit propagation. 

An important observation that can be made from the SAT based encodings of the cardinality constraint presented in the previous subsection, is that the obtained formula is horn.
Let us first introduce an important property, allowing us to grasp the intuition behind the encodings we propose in this paper. 

Let $\Phi$ be a horn formula, the sub-formula $\Phi^-$ denotes the set of negative clauses of $\Phi$ and $\Phi^+$ the set of clauses of $\Phi$ containing \textcolor{black}{exactly} one positive literal.

%=========
 \begin{proposition}
 \label{prop:hornP}
	Let $\Phi$ be a Horn formula and $\rho = \{x_1, x_2,\dots, x_k\}\subseteq V_{\Phi}$ an interpretation. $\Phi|_{\rho} \models^{*} \bot$ iff  $\exists c\in\Phi^-$ such that $\Phi^{+}|_{\rho} \models^{*} \bar{c}$.
	% (\Phi^-\backslash \{c\}))|_{\rho}\models c$
 \end{proposition}
 
 \begin{proof}
 ($\Rightarrow$)    Let us consider the formula $\Phi^{+}|_{\rho}$. Suppose that there is no clause $c \in \Phi^{-}$ such that $\Phi^{+}|_{\rho} \models^{*} \bar{c}$.
Let $S$ be the set of units literals of  $\Phi^{+}|_{\rho}$ including the literals of $\rho$. We can note that from $\Phi^{+}|_{\rho}$  only additional positive unit literals can be deduced by unit propagation ($S\supseteq \rho$). So $S$ is a set of positive literals.  $S  \cup  \overline{(V(\Phi) \setminus S)}$ is clearly a model of $\Phi|_{\rho}$. In fact, each clause of $\Phi^{+}|_{\rho}$ is a satisfied clause (its positive literal is in $S$) or contains at least one negative literal. Indeed, propagating positive literals over $\Phi^{+}$ leads to a formula where the remaining clauses contains a positive literal and at least one negative literal.  The remaining clauses of $\Phi^{-} \wedge  S$ are negative clauses before deleting from each clause the literals of $\bar{S}$. Then, by assigning the remaining variables $V(\Phi) \setminus S$ to $false$, we obtain a model $\mu = S\cup\overline{(V(\Phi) \setminus S)}$ of the formula $\Phi$. As $\rho\subseteq \mu$, this contradicts the assumption that $\Phi|_{\rho}$ is unsatisfiable. \\
 ($\Leftarrow$) From $\Phi^{+}|_{\rho} \models^{*} \bar{c}$, we have $(\Phi^{+}\wedge c)|_{\rho} \models^{*} \bot$. As $\Phi^{+}\wedge c\subseteq \Phi$, then $\Phi|_{\rho} \models^{*} \bot$. 
 \end{proof}
Given a horn formula $\Phi$, Proposition \ref{prop:hornP} expresses that  unsatisfiability under any interpretation made of a set of positives literals, is  caused by a clause from $\Phi^-$. As a cardinality constraint is usually encoded as a horn formula $\Phi$, to maintain GAC on the encoding of $y\rightarrow\Phi$, one only need to disjunctively add $\neg y$ to $\Phi^-$. 
 % the proposition provides a way for encoding conditional cardinality constraints.   
%We first start by considering the special case of atMostOne conditional cardinality constraint ($k=1$). 
%, then we provide  the encoding of the general case. 

\subsection{Conditional AtMostOne Constraint Encodings}
\label{sec:atMostOne}

Let us first consider the  conditional AtMostOne Constraint  $y \rightarrow \sum_{i=1}^{n} x_i \le 1$. Many encodings have been proposed to deal with the translation of AtMostOne constraint into CNF. 
Let us consider two standard encodings of this constraint. 

\subsubsection{Conditional AtMostOne Pairwise Encoding:}
The classical pairwise encoding can be obtained by considering the set of all binary negative clauses build over the set of variables $\{x_1, \ldots, x_n\}$ as described by the formula (\ref{form:AtMostOneP}). 
\begin{equation}
\label{form:AtMostOneP}
 \bigwedge_{1 \le i < j \le n} (\neg x_i \vee \neg x_j)
 \end{equation}
This naive formulation maintains generalized arc consistency and is in ${\cal O}(n)$ variables and ${\cal O}(n^2)$ clauses. The formula (\ref{form:AtMostOnePY}) encoding the conditional AtMostOne constraint $y \rightarrow \sum_{i=1}^{n} x_i \le 1$ is obtained  by simply adding $\neg y$ to all the clauses of the CNF formula (\ref{form:AtMostOneP}) obtained by pairwise encoding.   It is straightforward to remark that the obtained formula (\ref{form:AtMostOnePY}) allows to maintain generalized arc consistency. Indeed, any assignment of two literals $x_i$ and $x_j$ to true, allows to deduce $\neg y$ by unit propagation. On the other hand, if $y$ is assigned to $true$, the conditional constraint is reduced to a simple AtMostOne constraint which preserve generalized arc consistency. 

\begin{equation}
\label{form:AtMostOnePY}
 \bigwedge_{1 \le i < j \le n} ({\bf \neg y} \vee \neg x_i \vee \neg x_j)
 \end{equation}

\subsubsection{Conditional AtMostOne Sequential Counter \& Pigeon-Hole Encoding:}
The second encoding of the AtMostOne constraint is represented by formula (\ref{form:AtMostOneS}) obtained using sequential counter \cite{Sinz05}.
In \cite{JabbourSS14}, the authors shown that the same encoding is obtained using the pigeon-hole encoding described above and by applying an additional  step of variables elimination by resolution. In contrast to  pairwise encoding (\ref{form:AtMostOneP}), the one obtained by sequential counter (\ref{form:AtMostOneS}) is linear (${\cal O}(n)$ variables and clauses) thanks to the additional variables $\{p_1, \ldots, p_{n-1}\}$. Both encodings  (\ref{form:AtMostOneP}) and (\ref{form:AtMostOneS}) are known to maintain generalized  arc consistency.
\begin{equation}
\label{form:AtMostOneS}
\begin{array}{l}
  (\neg x_1 \vee p_1) \wedge ( \neg x_n \vee \neg p_{n-1}) \wedge \\
  \displaystyle \bigwedge_{1 < i < n} (\neg x_i \vee p_i) \wedge (\neg p_{i-1} \vee p_i) \wedge (\neg x_i \vee \neg p_{i-1}) 
 \end{array}
 \end{equation}
However, with the sequential counter based encoding, by adding $\neg y$ to all clauses of the formula (\ref{form:AtMostOneS}) we obtain a new formulation of the conditional AtMostOne constraint (formula (\ref{form:AtMostSY})) that does not maintain generalized arc consistency. 
\begin{equation}
\label{form:AtMostSY}
\begin{array}{l}
 ({\bf \neg y} \vee \neg x_1 \vee p_1) \wedge ( {\bf \neg y} \vee \neg x_n \vee \neg p_{n-1})  \wedge  \\
 \displaystyle \bigwedge_{1 < i < n} ({\bf \neg y} \vee \neg x_i \vee p_i) \wedge ({\bf \neg y} \vee \neg p_{i-1} \vee p_i) \wedge ({\bf \neg y} \vee \neg x_i \vee \neg p_{i-1}) 
\end{array}
 \end{equation}
Indeed, assigning two literals from $\{x_1, \ldots, x_n\}$ does not allow us to deduce $\neg y$ by unit propagation. For example, by assigning $x_1$ and $x_n$ to $true$, the two first clauses from (\ref{form:AtMostSY}) become binary. 
%Similarly if we assign any two literals $x_i$ and $x_j$ to $true$. 

To maintain the generalized arc consistency for the conditional AtMostOne constraint using sequential counter or pigeon-hole based encoding,  $\neg y$ must be added to a subset of the clauses as depicted in the formula (\ref{form:ImpAtMostOneUP}).
 
 \begin{equation}
\label{form:ImpAtMostOneUP}
\begin{array}{l}
 (\neg x_1 \vee p_1) \wedge ({\bf \neg y} \vee \neg x_n \vee \neg p_{n-1}) \wedge \\
  \displaystyle \bigwedge_{1 < i < n} (\neg x_i \vee p_i) \wedge (\neg p_{i-1} \vee p_i) \wedge ({\bf \neg y}  \vee  \neg x_i \vee \neg p_{i-1}) 
\end{array}
 \end{equation}
  \begin{proposition}
	The CNF formula (\ref{form:ImpAtMostOneUP}) encoding  $y \rightarrow \sum_{i=1}^{n} x_i \le 1$ using sequential counter or pigeon-hole encoding maintains generalized arc consistency by unit propagation. 
 \end{proposition} 
 \begin{proof}
 The proof of this proposition is a direct consequence of Proposition \ref{prop:hornP}. In fact, the encoding  of $\sum_{i=1}^{n} x_i \le 1$ is a horn formula. As a consequence when more than one literal from $\{x_1, \ldots, x_n\}$ are assigned to $true$, then a clause from the negative clauses of the encodings become $false$. Consequently, to encode $y \rightarrow \sum_{i=1}^{n} x_i \le 1$, it is sufficient to add $\neg y$ to the two negative clauses as shown in Constraint (\ref{form:ImpAtMostOneUP}). Indeed, suppose that we assign two arbitrary variables $x_i$ and $x_j$ (with $1<i<j<n$) to $true$. From the assignment of $x_i$ to $true$ and the clause $(\neg x_i\vee p_i)$, we deduce a unit literal $p_i$.
 Then, from the clause $(\neg p_i\vee p_{i+1})$ we deduce another unit literal $p_{i+1}$. This chain of  unit propagated literals continue until $p_{j-1}$.  Now if we assign $x_j$ to $true$, the clause $(\neg y \vee \neg x_j\vee \neg p_{j-1})$ allows us to deduce $\neg y$, as $p_{j-1}$ (propagated unit literal) and $x_j$ are assigned to $true$. Let us consider another case, where $x_1$ is assigned to $true$. Such assignment allows us to deduce thanks to unit propagation the literals $p_1, \dots, p_{n-1}$.  Then assigning any other literal $x_j$ (with $j\neq 1$), we deduce the literal $\neg y$, thanks to the clause $(\neg y\vee \neg x_j\vee \neg p_{j-1})$. Obviously assigning $y$ to $true$ leads to the classical encoding of the AtMostOne constraint which for the sequential counter and pigeon hole encoding preserve generalized arc consistency by unit propagation. 
 
 \end{proof}
 
 \subsubsection{Conditional AtMostOne Sorting Networks Encoding:}
 %\textcolor{blue}{Il reste a montrer conditional AtMostOne avec Sorting network}\\
 The sorting network encoding of the AtMostOne conditional constraint is similar to the conditional AtMostK constraint described in Section \ref{sec:sortedAtMk}. It is defined as: 
  $$\Phi_{SN}^{n, 1}(X;Z) \wedge ( {\bf \neg y} \vee \neg z_{2})$$ 
 
 Proposition \ref{prop:sortedAtMk} shows that the encoding, for any value of $k>0$, maintains generalized arc consistency by unit propagation. 

 \subsection{Conditional AtMostK Constraint Encodings}
 Let us now consider the general case of Conditional AtMostK Constraint.
 % of the form $y \rightarrow \sum_{i=1}^{n} x_i \le k$. 
\subsubsection{Pigeon-Hole based Encoding of Conditional Cardinaility:}
%In this section we deal with  efficient encodings of  conditional  cardinality constraint of the form  $y \rightarrow \sum_{i=1}^{n} x_i \le k$. This constraint can also be rewritten as  $\rightarrow \sum_{i=1}^{n} x_i > 1 \rightarrow \neg y$. 

% \begin{equation}
%\label{form:ImpAtMostOneUP}
% (\neg x_1 \vee p_1) \wedge (y \vee \neg x_n \vee \neg p_{n-1}) \wedge \bigwedge_{1 < i < n} (\neg x_i \vee p_i) \wedge (\neg p_{i-1} \vee p_i) \wedge (y \vee  \neg x_i \vee \neg p_{i-1}) 
% \end{equation}
 
% \begin{equation}
%\label{form:ImpAtMostOneUP}
% (\neg x_1 \vee p_1) \wedge (y \vee \neg x_n \vee \neg p_{n-1}) \wedge \bigwedge_{1 < i < n} (\neg x_i \vee p_i) \wedge (\neg p_{i-1} \vee p_i) \wedge (y \vee  \neg x_i \vee \neg p_{i-1}) 
% \end{equation}
% 

In Subsection \ref{sec:phc}, we reviewed the pigeon hole based encoding of the cardinality constraint  AtLeastK of the form $ \sum_{i=1}^{n} x_i \geq k$ proposed in \cite{JabbourSS14}. For clarity and consistency reasons, and as the constraint AtMostK $ \sum_{i=1}^{n} x_i \leq k$ can be equivalently rewritten as an AtLeastK constraint $\sum_{i=1}^{n} \neg x_i \geq n-k$,  for the pigeon hole based encoding, we consider the conditional AtLeastK constraint $y \rightarrow \sum_{i=1}^{m} x_i \geq k$. 

%
%encoding Once, we consider the case of the encoding based on Pigeon-Hole proposed in \cite{JabbourSS14} as reviewed before. As the pigeon-hole encoding deals with cardinality constraint of the form $ \sum_{i=1}^{n} x_i \le k$, the  Conditional AtMostK Constraint can be rewritten as $y \rightarrow \sum_{i=1}^{n} \neg x_i \leq n-k$
% 
%Adding  $\neg y$ to all the clauses of the resulting encoding of $\sum_{i=1}^{m} x_i \geq k$.
To preserve GAC, $\neg y$ must be added  to a limited subset of clauses of $ph{\cal P}_n^{k}$ encoding. Only the positives clauses of constraint (\ref{form:pigeonMP}) are augmented with $\neg y$.

\begin{equation}
\label{form:pigeonMPY}
 	\bigwedge_{1\le i \le k} (\neg y \vee \bigvee_{1 \le j < n-k+1} p_{ij})
 \end{equation}

 \begin{proposition}
	The encoding $(\ref{form:pigeonMPY}) \wedge  (\ref{form:pigeonXP}) \wedge (\ref{form:pigeonPP})$ preserves the generalized arc consistency of $y \rightarrow \sum_{i=1}^{n} x_i \geq k$.
 \end{proposition} 
 \begin{proof}
Let us note that the pigeon hole based encoding of $\sum_{i=1}^{n} x_i \geq k$ is a reverse-horn formula. So the Proposition \ref{prop:hornP} can be slightly modified to be adapted to the reverse-horn case by considering  assignments of variables to $false$ and positive clauses. As a consequence, one can conclude that adding $\neg y$ to the positive clauses is sufficient to maintain GAC by unit propagation.
 Let us sketch the proof using Example \ref{ex:ph}. The CNF encoding of the conditional constraint $y\rightarrow\sum_{i=1}^{6} x_i \geq 4$ is obtained from the CNF formula encoding $php{\cal P}_6^4$ by disjunctively adding $\neg y$ to the positive clauses (clauses on the right hand side). As we can observe the obtained formula remains in the reverse horn class. Let us show that by assigning any three variables among $x_1$ to $x_6$ to $false$, we deduce $\neg y$ by unit propagation. Suppose that $x_1$, $x_2$ and $x_3$ are assigned to $false$. From the second and third set of clauses, we deduce by unit propagation $\neg p_{11}$, $\neg p_{12}$, $\neg p_{21}$, $\neg p_{13}$, $\neg p_{22}$ and $\neg p_{31}$. Consequently, from the clause $(\neg y\vee p_{11}\vee p_{12}\vee p_{13})$, we deduce $\neg y$. Let us consider another case, say  $x_2$, $x_4$ and $x_6$ are assigned $false$. By unit propagation, we deduce $\neg p_{12}$, $\neg p_{21}$, $\neg p_{23}$, $\neg p_{32}$, $\neg p_{41}$, $\neg p_{43}$, $\neg p_{31}$, $\neg p_{41}$, $\neg p_{42}$. From the clause $(\neg y\vee p_{41}\vee p_{42}\vee p_{43})$, we deduce $\neg y$. Similarly, any other assignment of three variables from $x_1$ to $x_6$ produces $\neg y$ by unit propagation. 
 %As we can observe, the first two sets of clauses are reverse horn, by assigning the $x_i$ variables to false, we only deduce negative literals from the two first sets of clauses, leading to the propagation of $\neg y$ thanks to the third set of clauses (so the third set of clauses becoms 
 \end{proof}
   
\subsubsection{Sorted Networks based Encoding of Conditional Cardinality:}
\label{sec:sortedAtMk}
Let us now consider the sorted networks based encoding of the conditional AtMostK constraint $y \rightarrow \sum_{i=1}^{m} x_i \leq k $.  Using the sorted networks encoding of the AtMostK constraint (see Section \ref{sec:SNE}), its conditional variant can be represented by $y \rightarrow (\Phi_{SN}^{n, k}(X;Z) \wedge \neg z_{k+1})$  which is equivalent to the CNF formula $ ({\bf \neg y} \vee \Phi_{SN}^{n, k}(X;Z)) \wedge ( {\bf \neg y} \vee \neg z_{k+1} )$. 
As discussed in Section \ref{sec:SNE}, the basic comparator of two propositional variables, $2$-$ˆ'cmp(x_i, x_j ; z_i , z_j)$, is a building bloc of the sorted networks based encoding $\Phi_{SN}^{n, k}(X;Z)$, i.e., a conjunction of multiple formulas encoding two comparator basic components. 
%Let us note that the resulting encoding of $\Phi_{SN}^{n, k}(X;Z)$ is a conjunction of multiple basic comparator encodings.
 Consequently, the conditional formula $ (\neg y \vee \Phi_{SN}^{n, k}(X;Z))$ can be translated into CNF by adding $\neg y$ to all the clauses of each basic two comparators, which leads to multiple conditional two comparators of the form $ y \rightarrow  2$-$ˆ'cmp(x_i, x_j ; z_i , z_j)$, written in a clausal form as:
\begin{equation}
\label{condComp}
({\bf \neg y} \vee \neg x_i \vee z_i)\wedge ({\bf \neg y} \vee \neg x_j \vee z_i) \wedge ({\bf \neg y} \vee \neg x_i \vee \neg x_j \vee z_j)
\end{equation}

As we can see, assigning any input literal $x_i$ or $x_j$ to a conditional two comparators does not allow us to deduce any literal by unit propagation as all the clauses from (\ref{condComp}) become binary. In fact, to maintain  generalized arc consistency for the conditional AtMostK
constraint using sorting networks-based encoding, $\neg y$ must be disjunctively added only to the unit clause $\neg z_{k+1}$: 
%as described in the formula (\ref{condCompY}).
 \begin{equation}
 \label{condCompY}
   \Phi_{SN}^{n, k}(X;Z) \wedge ( {\bf \neg y} \vee \neg z_{k+1} )
 \end{equation}
 
\begin{proposition}
\label{prop:sortedAtMk}
	The encoding $\Phi_{SN}^{n, k}(X;Z) \wedge ( {\bf \neg y} \vee \neg z_{k+1} )$ preserves the generalized arc consistency of $y \rightarrow \sum_{i=1}^{n} x_i \leq k$.
 \end{proposition} 
 \begin{proof}
 In case where $y$ is assigned to $true$, the simplified formula represents the AtMostK constraint encoded using sorted networks.
 No, we consider two cases depending on the truth-value of $z_{k+1}$. In the first case, if $z_{k+1}$ is assigned to $true$, we deduce $\neg y$ by unit propagation. Indeed, as the outputs are sorted in descending order, this means that the AtMostK constraint is $false$, to satisfy the conditional AtMostK, one must assign $y$ to $false$.
In the second case, if the truth value of $z_{k+1}$ is $false$, this means that the AtMostK constraint is $true$, consequently, no matter is the value of $y$.  
 \end{proof}

\section{Sequential Unary Counter based Encoding of Conditional Cardinality Constraint}
%\textcolor{blue}{A faire}\\
%\textcolor{blue}{Indeed,  to preserve the GAC property for the conditional cardinality constraint $y\rightarrow \sum_{i=1}^{n} x_i \leq k$ using the previous encoding, we should encode $y\rightarrow (\ref{eq5}$) and keep the constraints ($\ref{eq1}), (\ref{eq2}), (\ref{eq3})$ and $(\ref{eq4})$. If $\neg y$ is added to all the constraint, we loose the unit propagation  after any affectation of the $x_i$ variables.}

We have shown in Subsection \ref{sec:atMostOne} how the conditional AtMostOne constraint can be encoded using the sequential counter-based encoding, while preserving the GAC property. Let us now consider the general case of Sequential counter-based encoding of conditional AtMostK constraint.
%The encoding $y\rightarrow (\ref{eq1})\wedge (\ref{eq2}) \wedge (\ref{eq3}) \wedge (\ref{eq4}) $ wich can be written as 
 The clauses $ (\ref{eq1}) \wedge (\ref{eq3}) \wedge (\ref{eq4})$ allow us to propagate any assignment of $x_i$ to synchronize all the intermediate sequential counters,  while the clauses   $ (\ref{eq2}) \wedge (\ref{eq5})$ allow us to detect any inconsistency of the constraint AtMostK. Indeed, by adding $\neg y$ to all the clauses, the literals $s_{ij}$ can not be propagated from any assignment of $x_i$ variables which prevent the synchronization operation. To preserve the GAC property, we should add $\neg y$ only to the clauses  of $(\ref{eq2}) \wedge (\ref{eq5})$ as shown in the following formula:
 \begin{equation}
\label{yeq1}
( \neg x_1 \vee s_{1,1})
\end{equation} 
\begin{equation}
\label{yeq2}
\bigwedge_{1 < j \leq k} ({\bf \neg y} \vee \neg s_{1,j})
\end{equation}  
\begin{equation}
\label{yeq3}
\bigwedge_{1 < i < n} ( \neg x_i \vee s_{i,1})\wedge ( \neg s_{i-1,1} \vee s_{i,1})
\end{equation} 
\begin{equation}
\label{yeq4}
\bigwedge_{1 < i < n} \bigwedge_{1<j\leq k} ( \neg x_i \vee \neg s_{i-1,j-1} \vee s_{i,j} )\wedge( \neg s_{i-1,j} \vee s_{i,j})
\end{equation}
\begin{equation}
\label{yeq5}
\bigwedge_{1 < i \leq n} ({\bf \neg y} \vee \neg x_i \vee \neg s_{i-1,k})
\end{equation}

 \begin{example}
 \label{ex:secount}
 Let us consider the following constraint $y\rightarrow x_1 +x_2+x_3 \leq 2$ which is encoded as follows: \\
 \noindent
 $(\neg x_1 \vee s_{1,1})\wedge 
({\bf \neg y} \vee \neg s_{1,2})\wedge
(\neg x_2 \vee s_{2,1})\wedge (\neg s_{1,1} \vee  s_{2,1})\wedge
(\neg x_2 \vee \neg s_{1,1} \vee s_{2,2})\wedge (\neg s_{1,2} \vee  s_{2,2})\wedge 
({\bf \neg y} \vee \neg x_2 \vee \neg s_{1,2})\wedge 
({\bf \neg y} \vee \neg x_3 \vee \neg s_{2,2})$\\
Assume that we start by assigning $x_1$ to $true$ then, the literals $s_{1,1}$ and $s_{2,1}$ are deduced by unit propagation. Next, if we assign $x_2$ to $true$, the literal $s_{2,2}$ is unit propagated. Finally by assigning $x_3$ to $true$ , which violates the constraint, the literal $\neg y$ is propagated thanks to the last clause.
\end{example}

\begin{proposition}
The encoding $(\ref{yeq1}) \wedge (\ref{yeq2}) \wedge (\ref{yeq3}) \wedge (\ref{yeq4})  \wedge (\ref{yeq5})$ preserves the generalized arc consistency of $y \rightarrow \sum_{i=1}^{n} x_i \leq k$.
\end{proposition}
  \begin{proof}
The encoding based on sequential counter of the cardinality constraint is also a horn formula. Consequently, we can apply the result of Proposition \ref{prop:hornP} to conclude that $\neg y$ must be added only to the negative clauses in order to preserve the generalized arc consistency. The proof is a simple generalization of those sketched in Example \ref{ex:secount}.
  \end{proof}

\section{SAT-based Association Rules Mining: A Case Study} 
We now present an application case, the problem of mining non-redundant association rules, whose encoding involves many conditional atMostOne constraints.  

\subsection{Association Rules Mining}
 Let $\Omega$ be a finite non empty set of symbols, called {\it items}. %From now on, we assume that this set is fixed.
 We use the letters $a$, $b$, $c$, etc. to range over the elements of $\Omega$.
 An {\it itemset} $I$ over $\Omega$ is defined as a subset of $\Omega$, i.e., $I\subseteq \Omega$. We use $2^{\Omega}$ to denote the set of itemsets over $\Omega$ and 
 we use the capital letters $I$, $J$, $K$, etc. to range over the elements of $2^{\Omega}$.
% \begin{itemize}
%\item  
A {\it transaction} is an ordered pair $(i, I)$ where $i$ is a natural number, called {\it transaction identifier}, and $I$ an itemset, i.e.,  $(i, I)\in \mathbb{N}\times 2^{\Omega}$. 
 A {\it transaction database} $\cal D$ is defined as a finite non empty set of transactions (${\cal D}\subseteq \mathbb{N}\times 2^{\Omega}$) where each transaction identifier  
 refers  to a unique itemset. % i.e., for all two transactions $(i,I)$ and $(j,J)$ in $\cal D$, $i=j$ implies $I=J$. 
  %\item 
 The {\it cover} of  an itemset  $I$  in a transaction database $\cal D$ is defined as 
 \textcolor{black}{${\cal C}over(I,{\cal D}) =  \{i\in \mathbb{N}\mid (i,J)\in {\cal D}~and~I\subseteq J\}$}. 
%\item 
The {\it support} of $I$ in $\cal D$ corresponds to the cardinality of  ${\cal C}over(I,{\cal D})$, i.e., $Support(I,{\cal D})= |{\cal C}over(I,{\cal D})|$. 
%\item 
An itemset $I\subseteq \Omega$ such that $Support(I,{\cal D})\geq 1$ is  a \emph{closed itemset} if, for all itemsets $J$ with $I\subset J$, 
$Support(J,{\cal D}) <Support(I,{\cal D})$. 

%\end{itemize}
\begin{example}
 Let us consider the transaction database $\cal D$ depicted in Table~\ref{tab:tabEx}. We have ${\cal C}over(\{c,d\},{\cal D})= \{1,2,3,4,5\}$ and 
 $Support(\{c,d\},{\cal D})= 5$ while $Support(\{f\},{\cal D})= 3$. The itemset $\{c, d\}$ is closed, while $\{f\}$ is not.
% When there is no ambiguity, we omit to mention the transaction database ${\cal D}$, and we simply note $Support(X\rightarrow Y)$ and ${\cal C}over(X\rightarrow Y)$. 
  \begin{table}[h!]
  {\scriptsize 
\begin{minipage}[t]{.47\linewidth}
   \begin{tabular}{c|c}
\hline
tid &  {Transactions}\\
\hline
1 & $~~~~~~C~~D~~E~~F~~G~~$\\
2 & $~~~~~~C~~D~~E~~F~~G~~$\\
3 & $A~~B~~C~~D~~~~~~~~~~~$\\
4 & $A~~B~~C~~D~~~~~F~~~~~$\\
5 & $A~~B~~C~~D~~~~~~~~~~~$\\
6 & $~~~~~~C~~~~~E~~~~~~~$\\
\hline
\end{tabular}
\caption{A Transaction Database $\cal D$}
\label{tab:tabEx}
\end{minipage}
\hfill
\begin{minipage}[t]{.53\linewidth}
   \begin{tabular}{c|l|c|c}
\hline
Name &  Asso. Rules & Support & Confidence\\
\hline
~~$r_1$~~ & $\{A\} \rightarrow \{B\}$ & 3/6 & 1\\
$r_2$ & $\{A\} \rightarrow \{B,C,D\}$ & 3/6 & 1\\
$r_3$ & $\{C\} \rightarrow \{D\}$ & 5/6 & 5/6\\
$r_4$ & $\{C,D\} \rightarrow \{E,F,G\}$ & 2/6 & 2/5\\
\hline
\end{tabular}
\caption{Some association rules}
\label{tab:tabExRules}\end{minipage}
}
\end{table}
\end{example}
%In this work, we are interested in  the problem of mining association rules (${MAR}$).
 An {\it association rule}  is a pattern of the form $X\rightarrow Y$ where $X$ (called the \textit{antecedent}) and $Y$ (called the \textit{consequent}) are two disjoint itemsets. The interestingness predicate is defined using the notions of support and confidence. 
% \begin{itemize}
%\item  
The {\it support of an association rule} $X\rightarrow Y$ in a transaction database $\cal D$, defined  as 
$Support(X\rightarrow Y) = \frac{Support(X\cup Y)}{|{\cal D}|}$, determines how often a rule is applicable to a given dataset, i.e., the occurrence frequency of the rule. 
%\item 
The {\it confidence} of $X\rightarrow Y$ in $\cal D$, defined as ${\cal C}over(X\rightarrow Y) = \frac{Support(X\cup Y)}{Support(X)}$, provides an estimate of the conditional probability of $Y$ given $X$. 
%\end{itemize}

A \emph{valid association rule} is an association rule with support and confidence greater than or equal to the minimum support threshold (minsupp) and minimum confidence threshold (minconf), respectively.  
%More precisely, given a transaction database $\cal D$, a minimum support threshold minsupp and a minimum confidence threshold minconf, 

\begin{definition} [Mining Association Rules Problem]
The problem of mining association rules consists in computing 
${\it MAR}({\cal D}, minsupp, minconf)= \{X\rightarrow Y\mid X,Y\subseteq {\Omega}, Support(X\rightarrow Y)\geq minsupp, {\cal C}over(X\rightarrow Y)\geq minconf\}$
\end{definition}

%\begin{example}

 %Table \ref{tab:tabExRules} illustrates some association rules with their supports and confidences. 
 %For instance, $Support(\{A\}\rightarrow \{B\})=\frac{3}{6}$ and  ${\cal C}over(\{A\}\rightarrow \{B\})=1$.
%\end{center}
%\end{example}

%-------------------------------------------------------------------------------------------

\subsection{SAT-based Non-Redundant Association Rules Mining}
 \label{sec:encodAssRul}
To mine association rules, Boudane et al. \cite{BoudaneJSS16}  proposed a SAT-based approach. Boolean variables are introduced to represent the antecedent and the consequence of an association rule $X \rightarrow Y$.  Support and confidence constraints are expressed as 0/1 linear inequalities over the  variables associated to transactions.

Let $\Omega$ be a set of $n$ items,  ${\cal D}=\{(1,I_1), \ldots{}, (m, I_m)\}$ a transaction database, where $\forall i\in\{1,m\}, I_i\subseteq \Omega$,  $minsupp$  (resp. $minconf$) a minimum support (resp. confidence) threshold. 
Each item is associated with two Boolean variables  $x_a$ and $y_a$. $x_a$ (resp. $y_a$)  is true if and only if $a\in X$ (resp. $a\in Y$). 
Similarly to \cite{JabbourSS13topk}, to represent the cover of $X$ and $X\cup Y$,  each transaction identifier $i\in\{1, m\}$  is associated with two propositional variables $p_i$ and $q_i$.  $p_i$ (resp. $q_i$) are used to represent the cover of $X$ (resp. $X\cup Y$). 
More precisely, given a Boolean interpretation $\cal I$, the corresponding association rule,  denoted $r_{\cal I}$,  is $X=\{a \in \Omega\mid {\cal I}(x_a)= 1\}\rightarrow Y= \{b\in \Omega\mid {\cal I}(y_b)= 1\}$, 
the cover of $X$ is $\{i\in \{1, m\}\mid {\cal I}(p_i)=1 \}$, and the cover of $X\cup Y$
is $\{i\in\{1,m\} \mid {\cal I}(q_i)=1 \}$.
The SAT encoding of the association rules mining problem is defined by the  constraints (\ref{form:0}) to (\ref{form:4bis}).  
%in a set of constraints defined below. 
  \begin{table}[h!]
\begin{minipage}[t]{.43\linewidth}
 \begin {equation}
 \label{form:0}
(\bigvee_{a \in \Omega} x_a)  \wedge (\bigvee_{a \in \Omega} y_a) %  \bigwedge_{a \in \Omega}
 \end {equation} 
 \begin {equation}
 \label{form:1}
\bigwedge_{a \in \Omega} (\neg x_a \vee \neg y_a) %  \bigwedge_{a \in \Omega}
 \end {equation}
  \begin{equation}
 \label{form:2}
 \bigwedge_{i\in 1..m} \neg p_i \leftrightarrow \bigvee_{a \in \Omega\setminus I_i} x_a %\bigwedge_{i=1}^m
 \end{equation}  
\end{minipage}
\hfill
\begin{minipage}[t]{.57\linewidth}
\begin{equation}
 \label{form:3}
  \bigwedge_{i\in 1..m} \neg q_i \leftrightarrow \neg p_i\vee (\bigvee_{a \in \Omega\setminus I_i} y_a)
 \end{equation}
\begin{equation}
   \label{form:4}
\sum_{i \in 1..m} q_i \geq m\times minsupp  
 \end{equation}
 \begin{equation}
   \label{form:4bis}
 \frac{\sum_{i \in 1..m} q_i}{\sum_{i\in 1..m} p_i} \geq minconf  
 \end{equation}
 \end{minipage}
\end{table}

The two clauses of  the formula (\ref{form:0}) express that $X$ and $Y$ are not empty sets. Formula (\ref{form:1})  allows to express $X\cap Y=\emptyset$. The formula  (\ref{form:2}) is used to represent the cover of the itemset $X$ corresponding to the left part of the candidate association rule. We know that the transaction identifier $i$ does not belong to ${\cal C}over(X,{\cal D})$ if and only if there exists an item $a\in X$ such that $a\notin I_i$. This property is represented by constraint (\ref{form:2}) expressing 
that  $p_i$ is $false$ if and only if $X$ contains an item that does not belong to the transaction $i$. In the same way, the formula (\ref{form:3}) allows to capture the cover of $X\cup Y$. To specify that the support of the candidate rule has to be greater than or equal to the fixed threshold $minsupp$ (in percentage), and the confidence is greater than or equal to  $minconf$ we use respectively the constraints (\ref{form:4}) and (\ref{form:4bis}) expressed by 0/1 linear inequalities. 

To extend the mining task to the \textcolor{black}{closed association rules}, the following constraint is added to express that  $X \cup Y$ is a closed itemset~\cite{JabbourSS13topk}: 
\begin{equation}
\label{form:closedC}
\bigwedge_{a\in\Omega} ( (\bigwedge_{i\in 1..m}  q_i \rightarrow a\in I_i)\rightarrow x_a \vee y_a)
\end{equation}
This formula means that, for all item $a\in \Omega$, if we have \textcolor{black}{$Support(X\cup Y, {\cal D})= Support(X\cup Y\cup\{a\}, {\cal D})$},  
which is encoded with the formula $\bigwedge_{i\in \{1, m\}}  q_i \rightarrow a\in I_i$, then we get $a\in X\cup Y$, which is encoded with $x_a \vee y_a$.

Several contributions deal with the enumeration of a compact representation of association rules. Among such representations, one can cite the well-known Minimal Non-Redundant Association Rules  \cite{Bastid00,Kryszkiewicz98MMR} defined as follows:
 
% \section{Minimal Non-Redundant Association Rules}
% \label{sec:MNR}
%In this section, we present our encoding of the problem of extracting non-redundant rules into propositional satisfiability. First, we focus on the interesting representation that corresponds to the minimal non-redundant association rules (MNRs in short)~\cite{Kryszkiewicz98MMR,Bastid00}.  
%

%It is worth noting that the number of association rules in a transaction database may be very large. As a consequence, several works consider different condensed representations to reduce the size of the output.
%In this work, we focus on the interesting representation that corresponds to the minimal non-redundant association rules (MNRs in short)~\cite{Kryszkiewicz98MMR,Bastid00}.  
\begin{definition}[Minimal Non-Redundant Rule]
\label{def:mnr}
%\cite{Kryszkiewicz98MMR}
An association rule $X\rightarrow Y$ is a minimal non-redundant rule  iff there is no association rule $X'\rightarrow Y'$ different from $X\rightarrow Y$ s.t. ($1$) $Support(X\rightarrow Y)=Support(X'\rightarrow Y')$, ($2$) $Conf(X\rightarrow Y)=Conf(X'\rightarrow Y')$ and ($3$) $X'\textcolor{black}{\subseteq} X$ and $Y\textcolor{black}{\subseteq} Y'$.
\end{definition}

\begin{example}
	Consider the  rules given in Table \ref{tab:tabExRules}. In this set of rules,   $\{a\} \rightarrow \{b,c,d\}$  is a minimal non-redundant rule while $\{a\} \rightarrow \{b\}$ is not.
\end{example}

Minimal non-redundant association rules are the closed rules in which the antecedents are minimal w.r.t. set inclusion.  The authors of \cite{Bastid00} provided a characterization of  the antecedents, called \textit{minimal generators}.

\begin{definition}[Minimal Generator] \label{def:generator}Given a closed itemset $X$. 
An itemset $X' \subseteq X$ is a minimal generator of $X$ iff $Support(X') = Support(X)$ and there is no $X'' \subseteq  X$ s.t.  $X'' \subset X'$ and  $Support(X'') = Support(X)$.
\end{definition}

 In  \cite{BoudaneJSS17}, the authors proposed to extend the SAT-based encoding to enumerate the minimal non-redundant association rules. To this end, the SAT-based encoding of association rules mining is enhanced with a Boolean constraint expressing that each antecedent is a minimal generator. This constraint expressing that  ${Supp}(X\rightarrow Y)= {Supp}(X \setminus  \{a\}\rightarrow Y) $ is defined as follows:

\begin{equation}
\label{form:generatorX}
(\bigwedge_{a \in \Omega} x_a \rightarrow \bigvee_{(i\in\{1\dots m\},~a \not\in I_i)}  ( \bigwedge_{b\notin I_i\cup \{a\}} \neg x_b))  \vee (\sum_{b\in \Omega} x_b = 1)     
\end{equation}

Using additional variables, this constraint is rewritten as:

%\begin{equation}
%\label{form:generatorBEST1}
%\begin{array}{l}
%[\displaystyle \bigwedge_{a \in \Omega} (x_a \wedge \rightarrow \bigvee_{(i\in \{1\dots m\},~a \not\in I_i)}  z_i )]\wedge
%[\displaystyle   \sum_{b\in \Omega} x_b \neq 1] \wedge [\bigwedge_{i\in 1..m} (z_i  \rightarrow \sum_{b\notin I_i} x_b \le 1) ]   
%\end{array}
%\end{equation}

\begin{equation}
\label{form:generatorBEST1}
\begin{split}
\displaystyle \bigwedge_{a \in \Omega} (x_a \wedge \neg z \rightarrow \bigvee_{(a \not\in I_i)}  z_i ) ~~\wedge~~ \bigwedge_{i\in 1..m} (z_i  \rightarrow \sum_{b\notin I_i} x_b \le 1)  ~~\wedge~~   z \rightarrow (\sum_{b\in \Omega} x_b = 1)
\end{split}
\end{equation}

%
%\begin{equation}
%\label{form:generatorBEST1}
%\begin{split}
%\displaystyle \bigwedge_{a \in \Omega} (x_a \wedge \neg z \rightarrow \bigvee_{(i\in \{1\dots m\},~a \not\in I_i)}  y_i ) \wedge\\
%\bigwedge_{i\in 1..m} (y_i  \rightarrow \sum_{b\notin I_i} x_b \le 1)  \wedge\\
% \displaystyle   z\rightarrow \sum_{b\in \Omega} x_b = 1
%\end{split}
%\end{equation}

%In fact,  this transformation comes from the fact that ($ \bigwedge_{b\notin I_i\cup\{a\}} \neg x_b$) is equivalent to  ($\sum_{b\notin I_i}  x_b \le 1$) in the case where  $I_i$ does not  contain $a$.

As we can observe, the previous constraint (\ref{form:generatorBEST1}) involves $m$ conditional AtMostOne constraints. We note $\Phi_{mnr}$, the conjunction of the formulas from (\ref{form:0}) to (\ref{form:closedC}) and (\ref{form:generatorBEST1}), encoding the problem of minimal non redundant association rules.
This encoding is used in our experimental evaluation to show the relevance of our proposed encoding. 

\section{Experiments} 

%\subsection{Experimental Context \& Protocol}
In this section, we consider the encoding of  minimal non-redundant association rules as described by the boolean formula $\Phi_{mnr}$ (see Subsection \ref{sec:encodAssRul}).
To enumerate the set of models of the resulting CNF formula, we follow the approach of \cite{BoudaneJSS16}. The proposed model enumeration algorithm  is based on a backtrack search DPLL-like procedure. In our experiments, the variables ordering heuristic, focus in priority on the variables of respectively $X$ and $Y$ to select the one to assign next. The main power of this approach consists in using watched literals structure to perform efficiently the unit propagation process.   Let us also note that the constraints (\ref{form:4}) and (\ref{form:4bis}) expressing respectively the frequency and the confidence are managed dynamically without translation into CNF form.
Indeed, these last constraints are handled and propagated on the fly as usually done in constraint programming.
Each model of the propositional formula encoding the association rules mining task corresponds to an association rule obtained by considering the truth values of the propositional variables encoding the antecedent ($X$) and the consequent ($Y$) of this rule. 

For illustration purposes, in our experiments, $\Phi_{mnr}^{gac}$ (resp. $\Phi_{mnr}^{n\_gac}$) denotes the formula $\Phi_{mnr}$, where the conditional AtMostOne constraints involved in the formula (\ref{form:generatorBEST1}) are expressed using the sequential counter encoding ((Section \ref{sec:atMostOne}) that maintains (resp. does not maintain) GAC property $gac$ (resp. $n\_gac$) expressed by the formula  (\ref{form:ImpAtMostOneUP}) (resp. \ref{form:AtMostSY}).  
%Similarly,  we note $\Phi_{mnr}^{gac}$ the formula $\Phi_{mnr}$, where the conditional  AtMostOne constraints involved in formula (\ref{form:generatorBEST1})  are encoded with GACuses an encoding that preserves GAC by using the encoding (\ref{form:ImpAtMostOneUP}). 
%Finally {\it SAT4MNR-GAC-CMP} is the encoding that uses an encoding that preserves GAC (\ref{form:ImpAtMostOneUP} and compact the set of conditional AtMostOne constraints.
In our experiments, for each data, the  support was varied from 5\% to 100\% with an interval of size 5\%. The confidence is varied in the same way. Then, for each data, a set of 400 configurations is generated.
All the experiments were done on Intel Xeon quad-core machines with 32GB of RAM running at 2.66 Ghz. For each instance, we fix the timeout to 15 minutes of CPU time.

%\subsection{Experimental Results}
%\paragraph{{\bf Results:}}

 Table \ref{tab:tabResultsMin} describes our comparative results.
 We report in column 1 the name of the dataset and its characteristics in parenthesis: number of items (\#items), number of transactions (\#trans) and density. For each encoding, we report the number of solved configurations ($\# S$), and the average solving time ($avg. time$ in seconds). For each unsolved configuration, the  time  is  set to 900 seconds  (time out). In the last row of Table \ref{tab:tabResultsMin}, we provide the total number of solved configurations and the global average CPU time in seconds. 
 
\begin{table}[!t]
\centering
{\scriptsize
\begin{tabular}{|l||r|r|r|r||r|r|r|}
\hline
 & \multicolumn{2}{c|}{~~$\Phi_{nmr}^{n\_gac}$~~} & \multicolumn{2}{c|}{~~$\Phi_{nmr}^{gac}$~~}  \\
 %& \multicolumn{2}{c|}{}&\multicolumn{2}{c|}{}\\
\hline\hline
&  	& avg.  &  	& avg.  \\
data  (\#items, \#trans, density) & ~~~\#S~~~	& ~~~time(s)~~~ &  ~~~\#S~~~	& ~~~time(s)~~~  \\ \hline
Audiology	(148, 216, 45\%) &	21 & 855.11 & {\bf 22} & {\bf 854.87}		\\ %\hline
Zoo-1   (36, 101, 44\%)	&  141 & 582.79 	& {\bf 400}	&	{\bf 0.25} 	 \\ 
%\hline
Tic-tac-toe (27, 958, 33\%)		&	395 & 12.7 & {\bf 400}	&	{\bf 0.16}  \\ 
%\hline
Anneal	 (93, 812, 45\%) &	20 & 855.00  &  {\bf 252}	&	{\bf 396.65} 		 \\ 
%\hline
Australian-credit (125, 653, 41\%) & 60 & 765.02 & 	{\bf 288} & {\bf 301.96 }	 \\ 
%\hline
German-credit (112, 1000, 34\%)	&	82 & 715.54 & {\bf 331 }& {\bf 203.508 }	 \\ 
%\hline
Heart-cleveland (95, 296, 47\%)	&	100 & 675.02 & {\bf 312} & {\bf 233.93}	 \\ 
%\hline
Hepatitis (68, 137, 50\%)	&	102 & 670.51 & {\bf 345} & {\bf 165.28}	 \\ %\hline
Hypothyroid (88, 3247, 49\%)	&	20 & 855.01 & {\bf 128} & {\bf 643.01}		 \\ %\hline
kr-vs-kp	 (73, 3196, 49\%) & 21 & 852.76 &	{\bf 173 }& {\bf 546.40}		 \\ %\hline
Lymph	(68, 148, 40\%)	&	 21 & 852.75 & {\bf 400} & {\bf 16.57} 	 \\ %\hline
Mushroom	 (119, 8124, 18\%)&	20 & 855.08 & {\bf 392} & {\bf 68.71}	 \\ %\hline
Primary-tumor (31, 336, 48\%)	&	144 & 577.05 & {\bf 400} & {\bf 3.87} 		\\ %\hline
Soybean (50, 650, 32\%)	& 63 & 758.26 &	{\bf 400} & {\bf 0.72 }		 \\ %\hline
%Splice-1	&	0	&	0	&	0	&	0	\\ \hline
Vote	  (48, 435, 33\%) &	243 	&	353.44 & {\bf 400} & {\bf 25.34} 		\\ %\hline
Splice-1	  (287, 3190, 21\%) &	363 	&	{\bf 90.68} & {\bf 380} & 168.83	 \\ \hline
\hline
Total	& 1816	&	645.42 &  {\bf 5023}	&	 {\bf 226.87} 	\\ \hline

\hline

\end{tabular}
}
\vspace{0.2cm}
\caption{{Comparative results: $\Phi_{nmr}^{n\_gac}$  vs $\Phi_{nmr}^{gac}$}}
\label{tab:tabResultsMin}
\end{table}

%{\it MNR rules}: 

 As we can observe,   the encoding $\Phi_{nmr}^{gac}$ clearly outperforms $\Phi_{nmr}^{n\_gac}$. Indeed, with the GAC encoding, we solve $3207$ more configurations than without maintaining GAC. The encoding $\Phi_{nmr}^{gac}$ is the best on all the data in terms of the number of solved configurations and average CPU time, except for {\it splice-1} data where $\Phi_{nmr}^{n\_gac}$ is better in term of CPU time, but not in the number of solved configurations.  Even if we considered a SAT-based data mining application  involving only conditional AtMostOne constraints, the performance improvements demonstrate the interest of our proposed encodings. 
 
% 
% With $\Phi_{nmr}^{gac}$ we solve more configurations than with $\Phi_{nmr}^{n\_gac}$. For instance, on {\it anneal} data, {\it SAT4MNR-GAC-COMP} is remarkably efficient.  It solves about $253$ configurations more than {\it SAT4MNR-NGAC} and $21$ configurations more than {\it SAT4MNR-GAC}. We can also remark that for {\it Lymph} data {\it SAT4MNR-GAC} and {\it SAT4MNR-GAC-COMP} solve all the configurations in an average time of $15s$ where {\it SAT4MNR-NGAC} cannot solve all the configurations and they take a lot of time compared to {\it SAT4MNR-GAC}. Interestingly enough, compressing the encoding of conditional AtMostOne constraints, allows to push further the performances of {\it SAT4MNR-GAC}. In fact,   {\it SAT4MNR-COMP} allows us to obtain better performances i.e., $52$ more solved  instances and  the average time solving is improved from $ 226.87s$ to $216.97s$.

\section{Conclusion and Future Works}
In this paper, we study the problem of encoding the conditional cardinality constraint of the form $y \rightarrow \sum_{i=1}^{n} x_i \le k$. 
In this context, we proved that most of the well known GAC encodings of the cardinality constraint need to be adapted in order to maintain such important property.
Our contributions suggested that the cardinality constraints should not be considered as an isolated single constraint, but as a part of more general formula.
More generally, encoding a complex constraint into CNF might consider its interaction with the whole formula where such constraint is involved.
As an application case, we considered the SAT-based encoding of the minimal non-redundant association rules mining problem, involving several conditional AtMostOne constraints.
Experimental results showed that a better encoding of  $y \rightarrow \sum_{i=1}^{n} x_i \le 1$ allows significant performances improvements. 
As a future work, we plan to extend  our framework to tackle other complex conditional constraints, including conditional 0/1 linear  inequalities.

\bibliographystyle{abbrv}
\bibliography{satBib} 
%
% Let us now that the problem raised in this paper can concern also constraint of the form $\sum_{i=1}^{n} x_i \le b \rightarrow \sum_{j=1}^{n} y_j \le b'$. By adding a new variable $z$ this constraint can be written as: 
% $$\sum_{i=1}^{n} x_i \le b \rightarrow z ~~~~, ~~~~ z \rightarrow \sum_{j=1}^{n} y_j \le b' $$  
%which can be rewritten as:
% $$ \neg z \rightarrow  \sum_{i=1}^{n} \neg x_i \le  n - b+1 ~~~~, ~~~~ z \rightarrow \sum_{j=1}^{n} y_j \le b' $$     
%

\end{document}